\documentclass[letterpaper, 10 pt, conference]{ieeeconf}

\usepackage{amsmath,amsfonts,bm}

\def\eqref#1{equation~\ref{#1}}

\def\1{\bm{1}}

\DeclareMathAlphabet{\mathsfit}{\encodingdefault}{\sfdefault}{m}{sl}
\SetMathAlphabet{\mathsfit}{bold}{\encodingdefault}{\sfdefault}{bx}{n}

\usepackage{graphicx}
\usepackage{hyperref}
\usepackage{url}
\usepackage[ruled,vlined,linesnumbered,noresetcount]{algorithm2e}
\usepackage{algpseudocode}

\usepackage{xcolor}

\newtheorem{lemma}{Lemma}
\newtheorem{theorem}{Theorem}
\newtheorem{corollary}{Corollary}
\newtheorem{assumption}{Assumption}

\newcommand{\norm}[1]{\left\lVert #1 \right\rVert}

\title{Inverse Reinforcement Learning from Non-Stationary Learning Agents}

\author{Kavinayan P. Sivakumar, Yi Shen, Zachary Bell, Scott Nivison, Boyuan Chen and Michael M. Zavlanos \thanks{This work is supported in part by AFOSR under award \#FA9550-19-1-
0169 and by NSF under award CNS-1932011.}
\thanks{K. Sivakumar is with the Department of Electrical and Computer Engineering, Duke University, Durham, NC 27708, USA {\tt\small kps29@duke.edu}}%
\thanks{Y. Shen and M. Zavlanos are with the Department of Mechanical Engineering \& Material Science, Duke University, Durham, NC 27708, USA {\tt\small yi.shen478@duke.edu, michael.zavlanos@duke.edu}}%
\thanks{Z. Bell and S. Nivison are with the US Air Force, AFRL Division, USA {\tt\small zachary.bell.10@us.af.mil, scott.nivison@us.af.mil}}
}

\begin{document}

\maketitle

\begin{abstract}
In this paper, we study an inverse reinforcement learning problem that involves learning the reward function of a learning agent using trajectory data collected while this agent is learning its optimal policy. To address this problem, we propose an inverse reinforcement learning method that allows us to estimate the policy parameters of the learning agent which can then be used to estimate its reward function. Our method relies on a new variant of the behavior cloning algorithm, which we call bundle behavior cloning, and uses a small number of trajectories generated by the learning agent's policy at different points in time to learn a set of policies that match the distribution of actions observed in the sampled trajectories. We then use the cloned policies to train a neural network model that estimates the reward function of the learning agent. We provide a theoretical analysis to show a complexity result on bound guarantees for our method that beats standard behavior cloning as well as numerical experiments for a reinforcement learning problem that validate the proposed method.
\end{abstract}

\section{Introduction} \label{introduction}

Understanding the intentions of autonomous agents \cite{machinetheoryofmind} has important implications in a variety of applications, ranging from cyber-physical systems \cite{marlselectiveoverviewtheories} to strategic games \cite{ddpgstarcrafttransfermicro},  \cite{multiobjcoopcoevmicrortsgames}. If an agent follows a fixed policy, its intentions remain consistent and can be learned by sampling a sufficiently large number of state-action pairs \cite{algs4irlngrussell}, \cite{algperspectiveimitationlearning}. However, in practice, sampling enough trajectories to learn a fixed optimal policy can be time consuming. Moreover, from an observer's perspective, it is not always easy to determine whether the observed agent's policy is fixed or not. Therefore, it is important to be able to estimate an agent's intentions while this agent is interacting with its environment to learn its optimal policy and before the policy has converged. 

Doing so using only state-action data is not straightforward. A common way to predict the intentions of an agent in deep reinforcement learning is assuming the policy weights are known. If the policy weights are known during learning, the distribution of agent actions at any given state can also be estimated \cite{surveymas}. Nevertheless, assuming that the structure of the agent's policy is known is not ideal in practice \cite{intrusiondetectionmultiagent}. Inverse reinforcement learning (IRL) instead aims to understand an agent's intentions by learning their reward function \cite{ogirlpaper}. A common approach in existing IRL literature is to use the maximum entropy to select a distribution of actions at a state that matches an expected distribution given observations of state-action pairs. By maximizing the entropy, one can account for all the sampled distributions of state-action pairs at highly visited states, thus making it more reliable to recover the reward function as no additional assumptions are made on the behavior at less visited states \cite{maxentropyirlbcexpertexamplesdazi}, \cite{maxentropypaper2singh}.

An important common assumption in the IRL literature discussed above is that the trajectories used to learn the reward functions are generated from a stationary expert policy. This assumption does not hold when trajectory data are collected as the agent interacts with its environment to learn its optimal policy. In this case, existing IRL methods are not straightforward to apply. To address this challenge, \cite{logelirl} proposes a way to learn an agent's reward function from data generated by this agent's non-expert policy while this policy is being updated and assuming that policy updates are done using stochastic gradient descent. However, the theoretical analysis of the approach in \cite{logelirl} requires that the reward function of the learner agent is approximated by a set of linear feature functions. In practice, selecting an expressive enough set of feature functions is challenging. For example, even in the case of radial basis functions that are commonly used as feature functions, choosing proper centers and bandwidth parameters is not straightforward  \cite{radialbasisfunctions}. In IRL, these parameters must be chosen to ensure that each state action pair has a unique feature and that no state action pairs are favored over others. If these criteria are not met, the learned reward function may be skewed towards the state action pairs that are favored more. 

Motivated by the approach in \cite{logelirl}, in this paper we propose an IRL method to learn the reward function of a learning agent as it learns its optimal policy. As in \cite{logelirl}, we assume that the agent updates its policy using stochastic gradient descent, but we do not assume that the agent's reward function is approximated by linear feature functions; it can be approximated by any nonlinear function, e.g., a neural network. Since the agent does not share its policy parameters explicitly, the challenge lies in estimating the agent's policy parameters as they are being updated during learning. To do so, we propose a new variant of behavior cloning, which we term bundle behavior cloning, that uses a small number of trajectories generated by the learner agent's policy at different points in time to learn a set of policies for this agent that match the distribution of actions observed in the sampled trajectories. These cloned policies are then used to train a neural network model that estimates the reward function of the learning agent. We provide error bounds of the policies learned by bundle behavior cloning and standard behavior cloning that allow us to choose an appropriate bundle size which validates bundle behavior cloning as a superior method for the problem we discuss. Moreover, we present numerical experiments on simple motion planning problems that show that the proposed approach is effective in learning the reward function of the learning agent. In comparison, empirical results presented in \cite{logelirl} assume that the agent's policy parameters are known so that behavior cloning is not needed to estimate them.

The rest of the paper is organized as follows. In Section \ref{problemdefinition}, we formulate the proposed IRL problem and introduce some preliminaries. In Section \ref{method}, we develop our proposed algorithm. In section \ref{complexityanalysis} we present theoretical results that support our proposed method. Finally, in Section \ref{experiments}, we numerically validate our method on simple motion planning problems.

\section{Problem Definition}
\label{problemdefinition}
Consider an agent with state space $\mathcal{S}$ and action space $\mathcal{A}$. We model this system as a finite Markov Decision Process (MDP), defined as $\mathcal{M} = (\mathcal{S}, \mathcal{A}, P, \gamma, R)$, where $P(S_{t+1}=s_{t+1} \vert S_t=s_t, A_{t}=a_t)$ is the transition function defined over the agent's state $s_t \in \mathcal{S}$ and action $a_t \in \mathcal{A}$, $\gamma \in (0, 1]$ is the discount factor, and $R(s_t, a_t, s_{t+1})$ is the reward received when the agent transitions from state $s_t$ to state $s_{t+1}$ using the action $a_t$. Let also $\pi_{\theta}(a_t \vert s_t) \rightarrow [0, 1]$, denote the agent policy, which is the probability of choosing an action $a_t$ at state $s_t$, and is parameterized by the policy parameter $\theta$. The objective of the agent is to maximize its accumulated reward, i.e.,
\begin{align}
    \underset{\theta}{\text{max }} J(\theta) = \underset{s_0 \sim \rho^u}{\mathbb{E}} [\sum_{t=0}^T \gamma^t R(s_t, a_t, s_{t+1})],
    \label{eq:objfunctionabcd}
\end{align}

\noindent where $\rho^u$ is the initial state distribution of the agent. In what follows, we assume that the agent does not share its policy parameters or rewards. Then, in this paper, we address the following problem:

\noindent \textbf{Problem 1:} (Inverse reinforcement learning from a learning agent) Learn the reward function of a learning agent using only trajectories of state action pairs $\tau = [(s_0, a_0), ..., (s_T, a_T)]$ generated as the agent interacts with its environment to learn its optimal policy $\pi_\theta$ that solves (\ref{eq:objfunctionabcd}).

To solve Problem 1 we make the following assumption.
\begin{assumption} \label{sgdassumption}
The agent updates its policy using the stochastic gradient descent (SGD) update $\theta_{t+1} = \theta_t + \alpha \nabla_\theta J$.
\end{assumption}
We note that many popular policy gradient methods use stochastic gradient descent, e.g., REINFORCE, Actor-Critic, etc. 
Specifically, in this paper, we use the REINFORCE algorithm to update the policy parameters of agent $i$ as
\begin{align}\label{eq:reinforceforward}
    \theta_{t+1} = \theta_t + \alpha \gamma^t R(s_{t+1}) \nabla \text{ln } \pi(a_t \vert s_t, \theta),
\end{align}

\noindent where $\theta$ is the policy parameter of the agent, $\alpha$ is the learning rate, and we denote $R(s_{t+1}) = R(s_t, a_t, s_{t+1})$. We call this the forward learning problem and make the following assumption on the learning rate. 
\begin{assumption} \label{learningrateassumption}
The learning rate of the agent $\alpha$ is known and is sufficiently small.
\end{assumption}

Finally, we make the following assumption on the information that is available in order to learn the reward of the agent.
\begin{assumption} \label{globalstateassumption}
All states and actions of the agent in the environment can be observed.
\end{assumption}

The main challenge with solving Problem $1$ is that without the policy parameters of the agent at every timestep, it is difficult to recover $R(s_{t+1})$. In the next section, we detail how we can recover $R(s_{t+1})$ using IRL and a new variant of behavior cloning that we propose, that we call bundle behavior cloning, to estimate the agent's policy. 

\section{Method}
\label{method}
Assuming it is known that the learning agent uses the update (\ref{eq:reinforceforward}) to update its policy parameters, to recover the reward at a state $s_{t+1}$ we need to know the policy parameters $\theta_t$ and $\theta_{t+1}$, as well as $\nabla \text{ln } \pi(a_t \vert s_t, \theta)$. As the policy structure, e.g., the structure of the policy neural network, and policy weights of the agent are considered unknown, it is not straightforward to estimate these policy parameters, or the gradient of the objective function. Here, we propose a novel variant of behavior cloning to learn these quantities, using only trajectory data collected during learning.

Behavior cloning is a supervised learning method used to learn a policy $\pi_\theta$ from a set of trajectories $\{\tau\}$ sampled from an expert policy. A number of applications ranging from autonomous driving \cite{bcautonomousdrivingcnn}, video games \cite{bcatarigamesvae}, and traffic control \cite{bctrafficcontrolops} have relied on behavior cloning to learn the desired expert policies. Generally, the assumption in these works is that the expert policy used to generate the data is stationary. However, if the expert policy changes as in this paper, it is not easy to assign trajectory data $\{\tau\}$ consisting of state and action pairs $\tau = \{(s_0, a_0), (s_1, a_1), ..., (s_T, a_T)\}$, to specific policies $\pi_{\theta}$. We address this challenge by instead bundling together trajectories and learning a single cloned policy for each bundle. Given Assumption \ref{learningrateassumption}, if the learning rate $\alpha$ is small enough, then consecutive policy parameters $\theta_{t+1}$ and $\theta_t$ will be close to each other. Therefore, for an appropriate bundle size, it is reasonable to expect that the trajectories in each bundle are generated by approximately the same policy. We call the proposed method Bundle Behavior Cloning, which we describe below.

We denote the cloned policy corresponding to the $k$th bundle $b_k$ in the set $[b_1,...,b_M]$ by $\pi^k_{\psi}$ where $\psi$ is the policy parameter. We assume that every bundle contains $B$ trajectories. Given a total number of training episodes $E$ from forward learning, the total number of bundles is defined as $M = E-B+1$ if a sliding window is used to construct the bundles, or $M = E / B$ rounded up to the nearest integer for non-overlapping sets of episodes. Here, the size of each bundle $B$ constitutes a hyperparameter. For each bundle $b_k$ consisting of $B$ trajectories, our goal is to generate a distribution of the agent's actions for each state denoted by $\rho_{b_k}(s)$ that approximates the distribution of the agent's actions in the policy we wish to clone. Then, to train the cloned policy we rely on the Mean Squared Error (MSE) loss
\begin{align}
    \text{loss}(\psi_k) = \sum_{s \in \mathcal{S}} \text{MSE}(\pi^k_{\psi}, \rho_{b_k}(s)).
\label{eq:bundlebc}
\end{align}
The proposed bundle behavior cloning method is outlined in Algorithm 1. Specifically, line 5 in Algorithm \ref{bbcalg} adds the loss over all the states in the state space so that the $k$th cloned policy parameter $\psi_k$ can accurately imitate the true agent policy parameters $\theta$ during the episodes. Fig. \ref{fig:bbcvisual} illustrates Bundle Behavior Cloning visually.
\begin{algorithm} \label{bbcalg}
\SetKwInOut{Input}{input}
\SetKwInOut{Output}{output}
\Input{List of trajectories of state-action pairs $\tau$ from entire learning, size of each bundle $B$, mean-squared error loss function $\text{MSE}$, uninitialized policy $\tilde{\pi}_\psi$}
\Output{Set of cloned policies \{$\tilde{\pi}^1_\psi$, $\tilde{\pi}^2_\psi$, ... $\tilde{\pi}^M_\psi$\}}
\BlankLine
\For{bundle index $k$ from $[1:M]$}
{
Initialize $\tilde{\pi}^k_\psi = \tilde{\pi}^{k-1}_\psi$\ or $\tilde{\pi}^k_\psi = \tilde{\pi}_\psi$ if $k=1$\;
Define $b_k$ as the trajectory set of episode indices from $(k-1)B$ (inclusive) to $kB$ (exclusive)\;
Calculate sampled distribution $\rho_{b_k}(s)$ of agent actions at state $s$ for bundle of trajectories $b_k$ for each $s \in \mathcal{S}$\;
Calculate the loss using (\ref{eq:bundlebc})\;
Update policy $\tilde{\pi}^k_\psi$ using backpropagation;
}

\caption{Bundle Behavior Cloning}
\end{algorithm}
\begin{figure}[h]
    \centering
\includegraphics[scale=0.25]{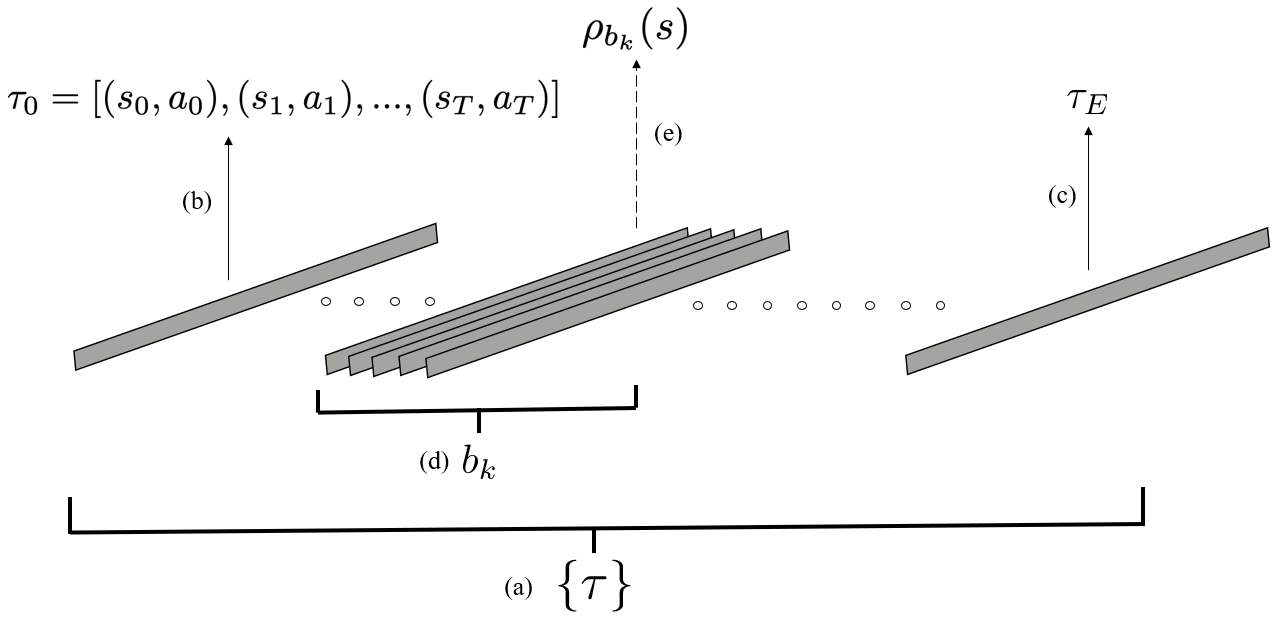}
    \caption{A visual illustration of Bundle Behavior Cloning. \textbf{(a)} Set of trajectories $\{\tau\}$ from total number of $E$ episodes during the forward step. \textbf{(b)} The first trajectory $\tau_0$ containing state action pairs for $T$ timesteps. \textbf{(c)} The last trajectory $\tau_E$ containing state action pairs for the last episode of learning. \textbf{(d)} The bundle $b_k$ of $M$ trajectories that are used to clone for a policy $\tilde{\pi}^k_\psi$. \textbf{(e)} The bundle of trajectories $b_k$ are sampled to get a distribution of individual agent actions per state $s$, $\rho_{b_k}(s)$.}
    \label{fig:bbcvisual}
\end{figure}

Using bundle behavior cloning, we can estimate the agent's policy parameters $\theta$ at various points in time during learning, which we can use to further estimate the gradient $\nabla_\psi \tilde{J}$ using the gradient of the cloned policy associated with bundle $b_k$. Given the cloned policy parameters and gradients, we can train a neural network $\beta(s_{t+1})$ to approximate the agent's reward function using the loss function obtained from (\ref{eq:reinforceforward}):
\begin{align}
\text{loss} = \sum_Z (\beta(s_{t+1})\nabla \pi^k_{\psi}(s_t, a_t)  \gamma^t - (\psi_{k+1} - \psi_k)),
\label{eq:losstrainbeta}
\end{align}
where $Z$ is the batch size of existing state action pairs from the forward case, $\nabla \pi^k_{\psi}(a_t|s_t)$ is the estimated gradient of the cloned policy from the $k$th bundle corresponding to the state action pair $(s_t, a_t)$, and $(\psi_{k+1} - \psi_k)$ is the difference in policy parameters between the $k$th bundle and the $k+1$ bundle. This loss is backpropagated through $\beta$ to predict rewards at states $s_{t+1}$ that will minimize the loss. The full algorithm for training $\beta$ is seen in Algorithm \ref{trainbetaalg} and the complete pipeline detailing all the steps is shown in Fig. \ref{fig:fullpipeline}.
\begin{algorithm} \label{trainbetaalg}
\LinesNumbered
\SetKwInOut{Input}{input}
\SetKwInOut{Output}{output}
\Input{List of trajectories of state-action pairs $\tau$ from entire learning, batch size $Z$, number of training episodes $F$, set of cloned policies of agent's true policies \{$\tilde{\pi}^1_\psi$, $\tilde{\pi}^2_\psi$, ... $\tilde{\pi}^M_\psi$\}}
\Output{Our neural network estimate of the agent's reward function $\beta(s)$}
\BlankLine
\For{$f$ in $F$}{
\For{$z$ in $Z$}{
Sample a state $s$ from $\tau$\;
Calculate the loss using (\ref{eq:losstrainbeta}) at $s$ and add it to overall loss\;
}
Update policy $\beta$ using backpropagation\;
}
\caption{Learning REINFORCE Reward Function of Agent}
\end{algorithm}


\begin{figure*}[h]
    \centering
\includegraphics[scale=0.24]{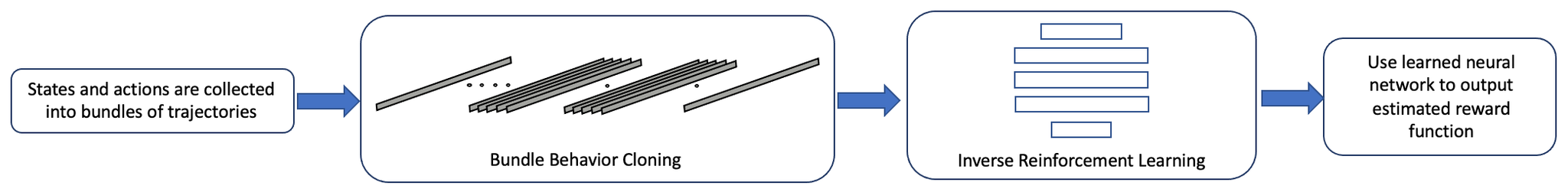}
    \caption{The full pipeline of the algorithm presented in this paper. Note that the learner does not need to finish optimizing its policy before its reward function can be estimated using the algorithm.}
    \label{fig:fullpipeline}
\end{figure*}

\section{Complexity Analysis}
\label{complexityanalysis}
In this section, we provide a sample complexity result and show that under certain conditions bundle behavior cloning achieves tighter complexity bounds than conventional behavior cloning. Specifically, given a set of trajectories of state-action pairs $\{\tau^1, ..., \tau^B\}$, where $\tau^i$ represents the trajectory under the fixed policy $\pi^i$, we show that bundle behavior cloning can be used to estimate the true policy $\pi^1(\cdot \vert s)$ of the agent at time step $1$ using $\hat{\pi}^{1:B}(\cdot \vert s)$, the estimated policy that best fits the distribution sampled from the bundle of trajectories $\{\tau^1, ..., \tau^B\}$ from timestep $1$ to timestep $BT$ and $T$ is the number of samples in one trajectory. Note that we only provide the analysis for $\pi^1$ and the results below hold for all time steps $t$.

Observe that if the polices in the bundle are significantly different from each other, we do not expect a larger bundle size to facilitate policy estimation. As a result, we make the following assumptions:
\begin{assumption} \label{consecutivebounds}
For every consecutive policy pair, there exists a constant $\epsilon$ such that for any state $s$, the following inequality holds:
\begin{align*}
       \vert \vert \pi^t(s) - \pi^{t+1}(s)\vert \vert_{\text{tv}} \leq \epsilon, \nonumber 
\end{align*}
\end{assumption}
where the total variation norm between two distributions $\mu$ and $\nu$ is defined as $TV(\mu, \nu) :=  \underset{A \in \mathcal{B}}{\text{sup}} \vert \mu(A) - \nu(A)\vert$ \cite{steinmethodapplications}, where $\mathcal{B}$ denotes the class of Borel sets.
%
%

The state visitation $d^\pi$ for a fixed policy $\pi$ is defined as $d^\pi = (1-\gamma) \sum_{t=0}^\infty \gamma^t \text{Pr}^\pi (s_t=s),$ where $\text{Pr}^\pi(s_t =s)$ is the probability that $s_t = s$ after starting at state $s_0$ and following $\pi$ thereafter.
The single policy behavior cloning sample complexity has been studied in \cite{agarwal2020flambe} and is presented below.
\begin{lemma}[Theorem 21 in \cite{agarwal2020flambe}] \label{bcstandardlemma}
With probability at least $1-\delta$, we have that
\begin{align}
    \mathbb{E}_{s\sim d^\pi} \vert \vert \hat{\pi}(\cdot \vert s) - \pi(\cdot \vert s) \vert \vert^2_{\text{tv}} \leq \frac{2 \text{ln}(\vert \Pi \vert / \delta)}{T} \nonumber,
\end{align}
where $T$ is the number of samples of state action pairs from a single trajectory, $\Pi$ is the policy class $\Pi = \{\pi : \mathcal{S} \rightarrow \Delta(\mathcal{A})\}$ which is discrete with size $\vert \Pi \vert$. 
\end{lemma}

Before applying Lemma \ref{bcstandardlemma} to $\tau^1$ (standard behavior cloning) and $\{\tau^1,\cdots,\tau^B\}$ (bundle behavior cloning), we assume the policies that generate $\{\tau^1,\cdots,\tau^B\}$ are independent. Then, defining the resultant cloned policies for both sets of trajectories $\hat{\pi}^1(\cdot \vert s)$ and $\hat{\pi}^{1:B}(\cdot \vert s)$ respectively, we have that, with probability $1-\delta$:
\begin{align} 
     &\mathbb{E}_{s\sim d^{\pi^1}} \vert \vert \hat{\pi}^1(\cdot \vert s) - \pi^1(\cdot \vert s) \vert \vert^2_{\text{tv}} \leq 2 \frac{\text{ln}(\vert \Pi \vert / \delta)}{T} \label{policy1expect} \\
     &\mathbb{E}_{s\sim d^{\pi^{1:B}}} \vert \vert \hat{\pi}^{1:B}(\cdot \vert s) - \pi^{1:B}(\cdot \vert s) \vert \vert^2_{\text{tv}} \leq \frac{2 \text{ln}(\vert \Pi \vert / \delta)}{BT}  \label{policyBexpect}.
\end{align}
In what follows, our goal is to provide a bound on $\mathbb{E}_{s\sim d^{\pi^{1}}} \vert \vert \hat{\pi}^{1:B}(\cdot \vert s) - \pi^{1}(\cdot \vert s) \vert \vert^2_{\text{tv}},$ which measures the difference between the policy returned by bundle behavior cloning and the true policy $\pi^1$. Note that, this expectation is taken with respect to the state visitation of $\pi^1$, which is different from (\ref{policyBexpect}). To address this challenge, we first present a result from \cite{achiam2017constrained} that shows that the difference between the distributions $d^{\pi^{1}}$ and $d^{\pi^{1:B}}$ can be bounded.
\begin{lemma}[Lemma 3 in \cite{achiam2017constrained}] \label{statedistributionlemma}
The divergence between discounted state visitation distribution is bounded by an average divergence of the policies:
\begin{align}
    \vert \vert d^{\pi'} - d^\pi \vert \vert_1 \leq \frac{2\gamma}{1-\gamma} \mathbb{E}_{s\sim d^\pi} \vert \vert \pi'(\cdot \vert s) - \pi(\cdot \vert s) \vert \vert_\text{tv},
\end{align} 
\end{lemma}
where $\gamma$ is the discount factor and $d^{\pi'}$, $d^\pi$ are two state visitation distributions associated with $\pi'(\cdot \vert s)$, $\pi(\cdot \vert s)$ respectively. 

Before proving our main theorem, we provide some properties of the policies within the bundle.
\begin{lemma} \label{thirdlemma}
Let Assumption \ref{consecutivebounds} hold. For any two policies $\pi^i, \pi^j \in \{\pi^1, ..., \pi^B\}$, we have that:
\begin{align}
    &\vert \vert \pi^i - \pi^j \vert \vert_\text{tv} \leq (B-1)\epsilon \nonumber\\
    &\vert \vert d^{\pi^i} - d^{\pi^j} \vert \vert_1 \leq \frac{2\gamma (B-1)\epsilon}{1-\gamma} \nonumber
\end{align}
\end{lemma}
\begin{proof}
The first property follows from the triangle inequality of total variations. According to Lemma \ref{thirdlemma}, we have $\norm{d^{\pi^i}-d^{\pi^j}}_1 \leq \frac{2\gamma}{1-\gamma}\mathbb{E}_{s\sim d^{\pi^j}}\norm{{\pi^i}(\cdot|s)-{\pi^j}(\cdot|s)}_{tv}\leq \frac{2\gamma}{1-\gamma}\mathbb{E}_{s\sim d^{\pi^j}}(B-1)\epsilon=\frac{2\gamma(B-1)\epsilon}{1-\gamma}.$ 
\end{proof}
\begin{corollary} \label{corollaryforproof}
Policies $\pi^1$ and $\pi^{1:B}$ satisfy the inequalities in Lemma $3$.
\end{corollary}
\begin{proof}
 We use the definition $\pi^{1:B} = \frac{1}{B} (\sum_{i=1}^B \pi^i)$ and substitute this into the left hand side norm of the first equation in Lemma \ref{thirdlemma} to get $\vert \vert \pi^1 - \frac{1}{B} \sum_{i=1}^B \pi^i\vert \vert = \frac{1}{B} (\sum_{i=2}^B \vert \vert \pi^1 - \pi^i \vert \vert)$. Using Lemma $3$ we can obtain an upper bound as $ \vert \vert \pi^1 - \pi^{1:B} \vert \vert_\text{tv} \leq \frac{1}{B} (\sum_{i=2}^B \vert i-1\vert ) \epsilon$. Since $\vert i-1\vert$ is upper bounded by $B-1$, we further get $ \vert \vert \pi^1 - \pi^{1:B} \vert \vert_\text{tv} \leq \frac{1}{B} (\sum_{i=2}^B B-1 ) \epsilon = \frac{(B-1)^2}{B}  \epsilon< (B-1)\epsilon$. We can then substitute this into Lemma \ref{statedistributionlemma} and obtain $\vert \vert d^{\pi^1} - d^{\pi^{1:B}} \vert \vert_1 \leq \frac{2\gamma (B-1)\epsilon}{1-\gamma}$. 
\end{proof}

We provide the sample complexity for bundle behavior cloning below.
\begin{theorem} \label{maintheorem}
With probability at least $(1-\delta)^2$, we have that 
\begin{align} \label{theoremequation}
    &\mathbb{E}_{s\sim d^{\pi^1}} \vert \vert \hat{\pi}^{1:B}(\cdot \vert s) - \pi^1(\cdot \vert s) \vert \vert^2_\text{tv} \nonumber\\ & \leq \frac{4\gamma(B-1)\epsilon}{1-\gamma} + \frac{4\text{ln}(\vert \Pi \vert /\delta)}{BT} + 2\epsilon^2 (B-1)^2.
\end{align}
\end{theorem}

\begin{proof}
  For simplicity of notation we hereby define $\rho^i = \pi^i(\cdot \vert s)$. First we rewrite the expectation on the left hand side of (\ref{theoremequation}) as $\mathbb{E}_{s\sim d^{\pi^1}} \vert \vert \hat{\rho}^B - \rho^B + \rho^B - \rho^1\vert \vert^2_\text{tv}$. Using the inequality $\vert \vert x - z\vert \vert^2 \leq 2 \vert \vert x - y\vert \vert^2 + 2 \vert \vert y - z\vert \vert^2$ we obtain an upper bound on the difference between the distribution of the true policy and that of the policy cloned from bundle behavior cloning as:
\begin{align*}
    \mathbb{E}_{s\sim d^{\pi^1}} \vert \vert \hat{\rho}^B - \rho^B + \rho^B - \rho^1\vert \vert^2_\text{tv} \leq \\
    2\mathbb{E}_{s\sim d^{\pi^1}} \vert \vert \hat{\rho}^B - \rho^B \vert \vert^2_\text{tv} + 2\mathbb{E}_{s\sim d^{\pi^1}} \vert \vert \rho^B - \rho^1 \vert \vert^2_\text{tv}.
\end{align*}
The right hand bound can then be transformed as:
\begin{align}
\begin{split}\label{eq:righthandboundtransform}
    2\mathbb{E}_{s\sim d^{\pi^1}} \vert \vert \hat{\rho}^B - \rho^B \vert \vert^2_\text{tv} - 2\mathbb{E}_{s\sim d^{\pi^B }} \vert \vert \hat{\rho}^B - \rho^B \vert \vert^2_\text{tv} \\+ 2\mathbb{E}_{s\sim d^{\pi^B}} \vert \vert \hat{\rho}^B - \rho^B \vert \vert^2_\text{tv} + 2\mathbb{E}_{s\sim d^{\pi^1}} \vert \vert \rho^B - \rho^1 \vert \vert^2_\text{tv}.
\end{split}
\end{align}
Using Corollary \ref{corollaryforproof} we obtain an upper bound on (\ref{eq:righthandboundtransform}):
\begin{align*}
    &\leq 2[\mathbb{E}_{s\sim d^{\pi^1}} \vert \vert \hat{\rho}^B - \rho^B \vert \vert^2_\text{tv} - \mathbb{E}_{s\sim d^{\pi^B }} \vert \vert \hat{\rho}^B - \rho^B \vert \vert^2_\text{tv}] \\
    & \quad + \frac{4\text{ln}(\vert \Pi \vert /\delta)}{BT} + 2\epsilon^2 (B-1)^2\\
    &\leq 2[\sum_s (\mathbb{P}_{d^{\pi^1}}(s) - \mathbb{P}_{d^{\pi^B}}(s)) \vert \vert \hat{\rho}^B - \rho^B \vert \vert^2_\text{tv}] \\
    &\quad + \frac{4\text{ln}(\vert \Pi \vert /\delta)}{BT} + 2\epsilon^2 (B-1)^2\\
    &\leq 2[\sum_s \vert \mathbb{P}_{d^{\pi^1}}(s) - \mathbb{P}_{d^{\pi^B}}(s)\vert]+ \frac{4\text{ln}(\vert \Pi \vert /\delta)}{BT} + 2\epsilon^2 (B-1)^2.
\end{align*}
Using Lemma \ref{statedistributionlemma}, we further have that:
\begin{align}\label{finalinequality}
    &\leq \frac{4\gamma(B-1)\epsilon}{1-\gamma} + \frac{4\text{ln}(\vert \Pi \vert /\delta)}{BT} + 2\epsilon^2 (B-1)^2
\end{align}
which completes the proof   
\end{proof}

Note that with the same probability $(1-\delta)^2$, we have with probability at least $(1-\delta)^2$, $\mathbb{E}_{s\sim d^{\pi^1}} \vert \vert \hat{\pi}^1(\cdot \vert s) - \pi^1(\cdot \vert s) \vert \vert^2_{\text{tv}} \leq \frac{2 \text{ln}(\vert \Pi \vert / (2\delta - \delta^2))}{T}$ according to (\ref{policy1expect}). Note also that there exists an optimal $B$ that minimizes the upper bound in (\ref{finalinequality}). If the minimum value of (\ref{finalinequality}) is strictly less than $\frac{2 \text{ln}(\vert \Pi \vert / (2\delta - \delta^2))}{T}$ as derived from (\ref{policy1expect}), then Theorem 1 guarantees that the policy learned from the bundle achieves a tighter bound on the difference between the distribution of the cloned policy and that of the true policy than the same bound achieved in standard behavior cloning. Of course, the minimum value of the bound in (\ref{finalinequality}) also depends on $\epsilon$, so bundle behavior cloning may perform worse than conventional behavior cloning if $\epsilon$ is large.

\section{Experimental Results}
\label{experiments}
\label{experiments}
\subsection{Testing Environment}
In this section, we demonstrate the proposed IRL algorithm on a Gridworld environment, where the learning agent seeks to reach a goal state. The environment is a $7$ by $7$ grid, with the agent starting at the top left square. It has three possible actions, moving to the right, moving to the left, and moving down. The agent receives the maximum reward of $20$ when it reaches the goal state, located at the bottom right square in the grid, but different squares on the grid result in different, negative rewards for the agent ranging from $-5$ to $-2$. All experiments are run using PyTorch \cite{paszke2017automatic} on a Windows system with an RTX 3080 GPU.

\subsection{Learning Rewards}
\begin{figure}
    \centering
\includegraphics[scale=0.3]{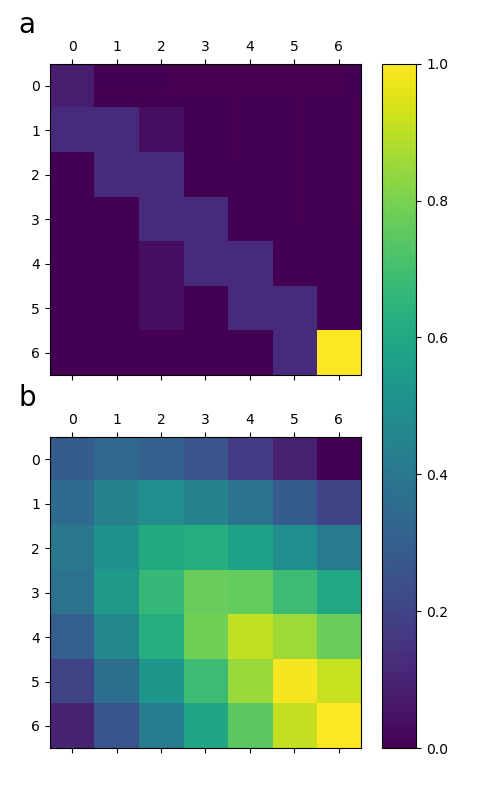}
    \caption{Colormaps representing the true and learned reward functions using Algorithm $2$ with Bundle Behavior Cloning \textbf{(a)} Agent's normalized, true reward function. \textbf{(b)} Agent's normalized, learned reward function with the same neural network structure.}
    \label{fig:colormapbbc}
\end{figure}
We use our proposed IRL method with bundle behavior cloning to learn the reward function of the agent. The forward policy is modeled by a 2 layer neural network with $16$ hidden nodes and ReLU activation functions on the first layer. The neural network used to clone the policies is the same as that in the forward case. The bundle has size $B=15$. There are $M=333$ total bundles defined across $5000$ trajectories and each trajectory consists of $15$ timesteps. As a result, each bundle contains $225$ state action pairs. The learning rate used during training of the forward and cloned policies is $0.00075$ and $\gamma$ is set to $0.999$. The neural network used to model the estimated reward $\beta$ is a 2 layer neural network with $20$ hidden nodes with a ReLU activation function on the first layer. The batch size used to train $\beta$ is $100$ and the total number of training episodes to is $5000$.

As neural networks' weights are trained stochastically \cite{uncertaintyneuralnetworks}, we use only the last layer of weights from the cloned policies to train $\beta$ in order to minimize the variance in estimated gradients from layers before the output layer. The last layer in a neural network is used to output the distribution of actions, and thereby we can minimize the effect of multiple combinations of neural network weights that could output the same estimated distribution.

To compare the reward $\beta$ learned using bundle behavior cloning and the proposed IRL method to the true agent reward, we first normalize and scale the reward $\beta$. Fig. \ref{fig:colormapbbc} shows the comparison results on a normalized color map. The recovered reward function is able to localize the goal state in the environment as well as identify states on the optimal path.

To further validate the predicted rewards, we use them to train a new policy which we then compare to the optimal forward policy learned from the true reward function. The average rewards and standard deviations of these forward and inverse learned policies are shown in Table \ref{maintable}. The learned policy is able to produce the same optimal trajectory as the forward optimal policy and achieves the same possible reward. 

\begin{figure*}[h]
    \centering
\includegraphics[scale=0.32]{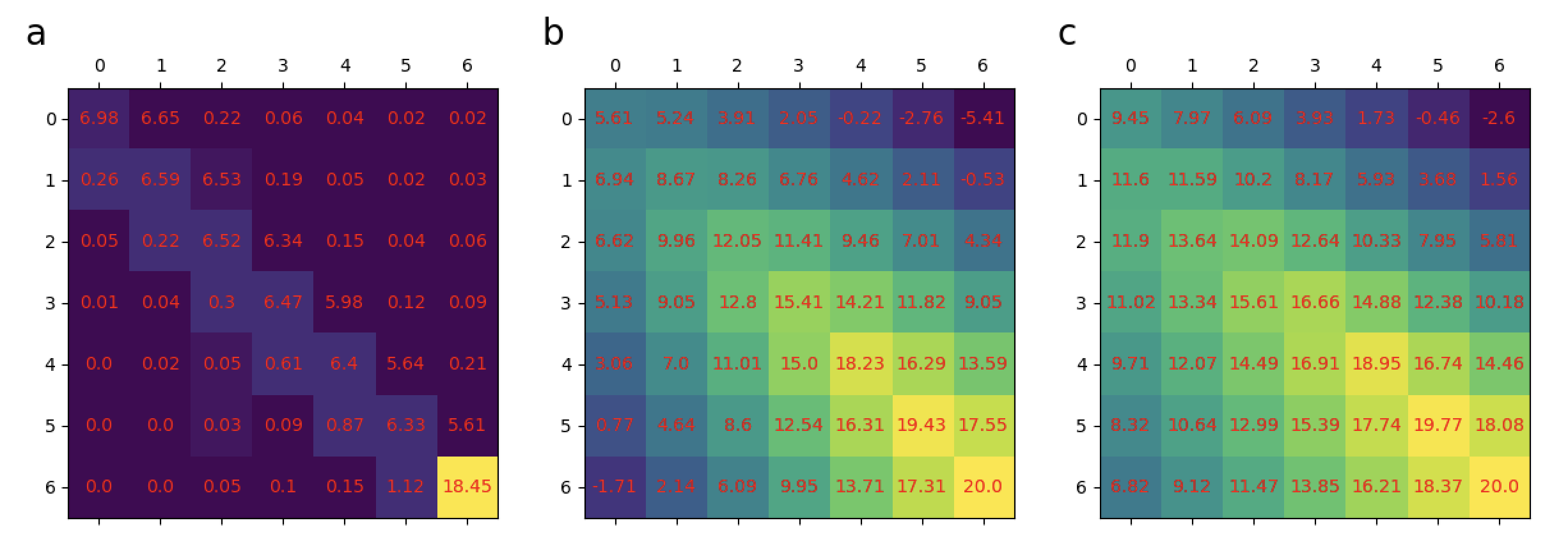}
\caption{The state distribution in the forward case can affect the variance in reward estimation. \textbf{(a)} The state distribution of the forward trajectories in percentages. \textbf{(b)} Lower bound and \textbf{(c)} upper bound of predicted scaled, normalized rewards for $95\%$ confidence measure. States with higher percentages of visitation correlate with less variance in estimated rewards.}
    \label{fig:95confidence}
\end{figure*}
In Fig. \ref{fig:95confidence} we train the neural network $\beta$ $10$ times. We normalize and scale the results as before, and present the lower and upper bound reward functions for a $95\%$ confidence interval. In Fig. \ref{fig:95confidence}a we see the state distribution from the forward learning trajectories. States that are visited more have a smaller range of estimated rewards within this confidence interval, whereas states like the bottom left corner have the largest range.

\begin{table}
  \caption{Averaged Rewards from learned Policy}
  \label{sample-table}
  \centering
  \begin{tabular}{l|l}
    \hline
    Description     & Average Reward    \\
    \hline
     Forward Policy: 2 layer, 16 hidden nodes & \textbf{57.461 $\pm$ 4.199}       \\
     BBC: 2 layer, 16 hidden nodes (same structure)     & \textbf{57.840 $\pm$ 0.946}     \\
     BBC: 2 layer, 16 hidden nodes, independent policies & \textbf{54.551 $\pm$ 1.777}     \\
    \hline
    BBC: 2 layer, 8 hidden nodes & \textbf{57.923 $\pm$ 1.248}      \\
    BBC: 2 layer, 24 hidden nodes     & \textbf{57.912 $\pm$ 1.262}       \\
    BBC: 2 layer, 32 hidden nodes     & 46.855 $\pm$ 3.865   \\
    BBC: 3 layer, 16 hidden nodes     & \textbf{52.190 $\pm$ 3.037}     \\
    BBC: 5 layer, 16 hidden nodes     & \textbf{54.934 $\pm$ 1.304}\\
    \hline
  \end{tabular}
   \label{maintable}
\end{table}

\subsection{Independent Policies}
In the forward case, future policy weights depend on previous ones, so the assumption that the policies that generate $\{\tau^1,\cdots,\tau^B\}$ are independent in the theoretical analysis in Section \ref{complexityanalysis} may not necessarily hold in practice. Simulating independent policies is possible as shown in \cite{simulatingiidpoison}, \cite{simulatingiidfromarbitrary}, and we do so by doubling the size of the bundle but skipping every other state action pair during sampling. As a result, the number of state action pairs per bundle is kept constant and subsequent policies are less dependent on each other. This sampling strategy still allows to reach the optimal goal position as seen in the first part of Table \ref{maintable}.

\subsection{Testing Different Neural Network Structures}
We test different neural network structures to measure robustness of bundle behavior cloning to different network structures when the true neural network structure of the agent's forward policy is not known. The second part of Table \ref{maintable} shows the average rewards of the learned policy for different neural network structures. In our experiments, we assume we know the activation functions used in between layers (ReLU). We leave the analysis of the impact of different activation functions on the performance of the algorithm to future work. We observe that it is possible to still learn the reward function of the agent even if a different neural network structure is used, compared to that of the agent's forward policy. Nevertheless, large deviations between the learnt and true policy networks have an effect on the reward estimation; the 32 hidden nodes network, results in a larger error in reward estimation.

\begin{figure}[h]
    \centering
\includegraphics[scale=0.45]{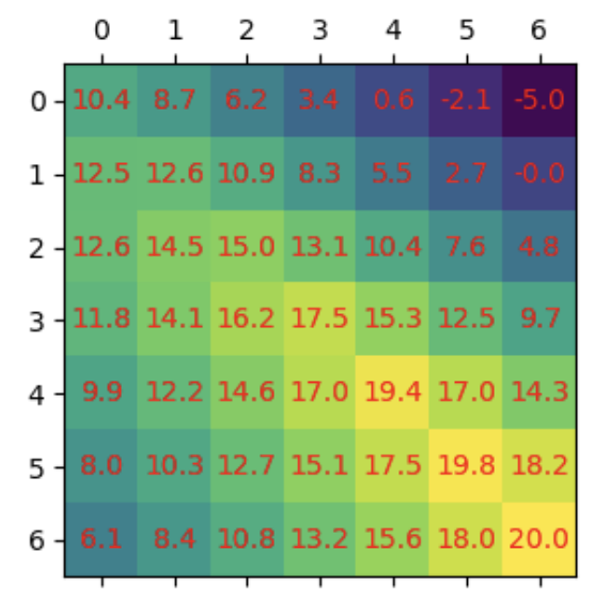}
    \caption{Only using the first 50 bundles ($750$ episodes) to learn the reward function.}
    \label{fig:diffbundlesizes}
\end{figure}
\subsection{Learning While Learning}
In this experiment, we test the ability of our method to learn the reward function of the learning agent while it is interacting with its environment to learn its optimal policy. Specifically, we test our method with just the first $50$ bundles rather than the full set of $333$ bundles. In Fig. \ref{fig:diffbundlesizes}, we show the recovered reward function using these $50$ bundles, which compose only $750$ episodes from the forward learning. We see that with less than a sixth of the total episodes from the forward learning, we can start to recover the shape of the reward function. We conclude that, compared to regular behavior cloning that requires an optimal policy to imitate, bundle behavior cloning does not.

\begin{figure}
    \centering
\includegraphics[scale=0.40]{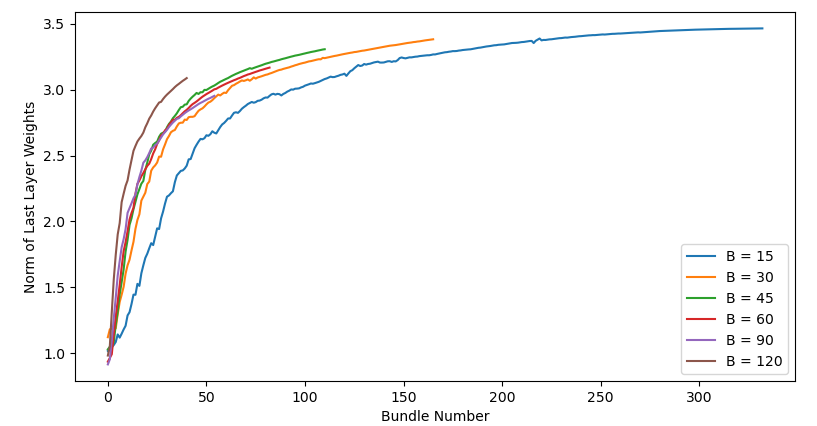}
    \caption{We see that the higher the bundle size $B$ is, the faster the last layer of weights in the cloned policies change across bundles. We've chosen $B=15$ in our experiments as a good balance between allowing enough samples for each bundle but also keeping Assumption $4$.}
    \label{fig:normslastweights}
\end{figure}

\subsection{Deviation of Cloned Policies in Between Bundles}
Finally, we empirically test the effect of the bundle choice on the performance of the proposed method. Fig. \ref{fig:normslastweights} shows the norm of the last layer weights of the cloned policies across bundles for different bundle sizes. As the bundle size gets larger, the weights of the last layer change much faster across the bundles, which may violate Assumption $4$. Although the norm of the weights of a neural network is not a substitute for the distribution of the policy itself, this graph does show that adjusting bundle size $B$ can help satisfy the theoretical requirement for $\epsilon$ in Assumption \ref{consecutivebounds} as it results in policy networks that are closer to each other across consecutive bundles.


\bibliography{iclr2023_conference}
\bibliographystyle{IEEEtran}

\end{document}